\documentclass{sig-alternate-05-2015}
\usepackage{amsmath}
\usepackage{amsfonts}
\usepackage[longend,boxruled]{algorithm2e}
\usepackage{algorithmic}
\usepackage{hyperref}
\usepackage{graphicx} 
\usepackage{subfigure}
\usepackage{float}
\usepackage{xfrac}
\usepackage{balance}
\usepackage{wrapfig,booktabs}
\usepackage[toc,page]{appendix}

\newcommand\numberthis{\addtocounter{equation}{1}\tag{\theequation}}
\usepackage{array}
\usepackage{rotating}
\usepackage{multirow}
\usepackage{booktabs}
\newcolumntype{L}[1]{>{\raggedright\let\newline\\\arraybackslash\hspace{0pt}}m{#1}}
\newcolumntype{C}[1]{>{\centering\let\newline\\\arraybackslash\hspace{0pt}}m{#1}}
\newcolumntype{R}[1]{>{\raggedleft\let\newline\\\arraybackslash\hspace{0pt}}m{#1}}

\begin{document}

\setcopyright{acmcopyright}

\title{Fast Collaborative Filtering from Implicit Feedback with Provable Guarantees}

%
%
%
%
%

\author{Sayantan Dasgupta}

\maketitle
\begin{abstract}
Building recommendation algorithms is one of the most challenging tasks in Machine Learning. Although most of the recommendation systems are built on explicit feedback available from the users in terms of rating or text, a majority of the applications do not receive such feedback. Here we consider the recommendation task where the only available data is the records of user-item interaction over web applications over time, in terms of subscription or purchase of items; this is known as implicit feedback recommendation. There is usually a massive amount of such user-item interaction available for any web applications. Algorithms like PLSI or Matrix Factorization runs several iterations through the dataset, and may prove very expensive for large datasets. Here we propose a recommendation algorithm based on Method of Moment, which involves factorization of second and third order moments of the dataset. Our algorithm can be proven to be globally convergent using PAC learning theory. Further, we show how to extract the parameters using only three passes through the entire dataset. This results in a highly scalable algorithm that scales up to million of users even on a machine with a single-core processor and 8 GB RAM and produces competitive performance in comparison with existing algorithms.
\end{abstract}


\section{Introduction}

Recommendation Systems came into the spotlight through the Netflix One-Million challenge. Most of the early recommendation systems were built using features extracted from the content of the items. These are known as content-based recommendation systems, and they typically fail to capture the user opinion. Collaborative filtering was introduced to mine user feedbacks to overcome the limitation of content-based filtering. Collaborative filtering mostly relies on the availability of user feedback, either in the form of numeric rating, or text, or even through binary 'like' or 'unlike' tags. However, not all applications receive such explicit feedback from users. 

Most of the web-based applications receive a significant amount of user traffics. The users interact with different items in the web applications, although they may not always rate the items. The web usage data containing user-item interactions can effectively be mined to build recommendation systems. Also, in applications where a user provides rating or feedback, such as Netflix, he/she rates only a small subset of movies watched. A user may simply avoid rating some of his favourite movies due to the lack of time, and there is no way to know about his interest in those movies except web-usage data of implicit user-item interaction. Also, the amount of web usage data for such applications is far larger than the amount of rating data available from users, and mining these data can provide an improvement on recommendations drawn only from user ratings. Please note that binary 'like' or 'dislike' tags provided by users are a form of explicit feedback, such as the case of \cite{wang2011collaborative}. We do not attempt to build a recommendation algorithm based on user tags here. An appropriate visualization of our recommendation problem in the line of  \cite{wang2011collaborative} is shown in Figure \ref{fig:demo}.

\begin{figure}[tb]
\centering
\includegraphics[scale=0.2]{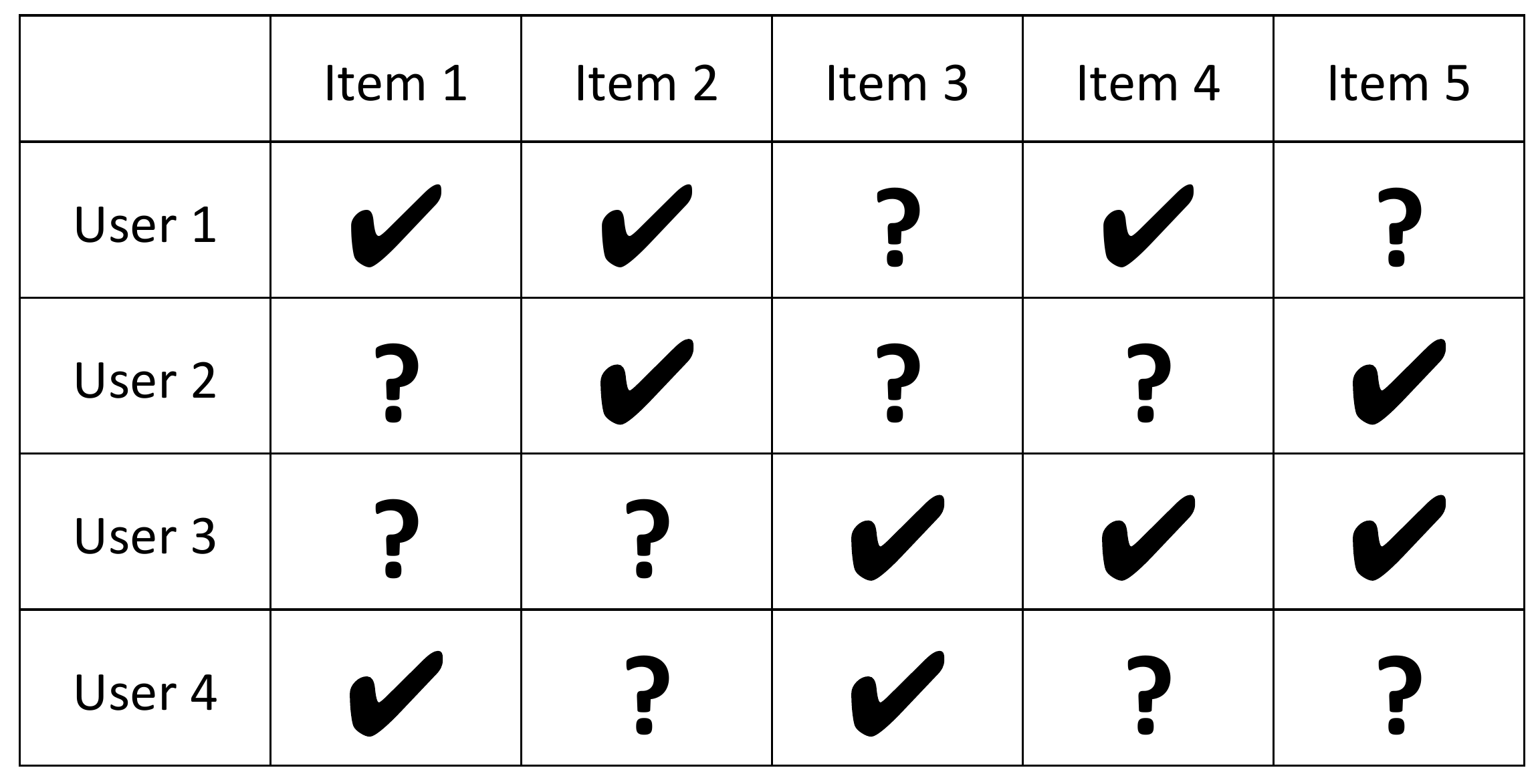} 
\caption{Implicit User-Item Interaction}
\label{fig:demo}
\end{figure}

The most common algorithm used by practitioners to build recommendation systems based on implicit feedback is Probabilistic Latent Semantic Indexing (PLSI), such as in personalized ranking of search results \cite{lx_personalized} or personalized news recommendation \cite{das_personalized}, However, PLSI trains using EM algorithm that suffers from local maxima problem. Therefore, these recommendation systems more often or so do not give optimal performance. Recent literature on recommendation systems includes different algorithms for implicit feedback dataset, although most of them are tested on datasets of limited size. \cite{WRMF} adapts the well-known matrix factorization algorithm for implicit feedback datasets through a weighted matrix factorization (WRMF). The algorithm scans through the entire dataset during every iteration until convergence, and it may prove computationally very expensive for a large volume of user logs stored across multiple nodes in a distributed ecosystem. Bayesian Personalized Ranking (BPR) \cite{BPR} uses a stochastic approach to sample negative items for each user, and reduces the computation time significantly. There are other algorithms in the literature, which are extensions of these matrix factorization methods. GBPR\cite{GBPR} builds on BPR and incorporates group preference into it. LorSLIM \cite{LORSLIM} uses a low rank sparse linear method for implicit feedback datasets. AdaBPR\cite{AdaBPR} introduces a boosting technique to improve on BPR loss. These algorithms are found to outperform other methods such as similarity or neighbourhood based methods.

There have been recent developments in non-iterative learning algorithm based on Method of Moments (MoM) \cite{MoM}, also referred to as Spectral Methods in the literature. Unlike traditional clustering algorithms that try to maximize likelihood or minimize cost through iterative steps, MoM attempts to learn the parameters through factorization of higher order moments of the data. It is a non-iterative algorithm and offers much better scalability than iterative counterparts, especially for large datasets. Here we use Method of Moments on the same generative latent variable used by PLSI \cite{PLSI}, and show how to extract the parameters through factorization of moments of the data. We demonstrate the derivation of our algorithm in next section, prove its convergence bounds, and then compare the performance of our algorithm with PLSI and matrix factorization on real-life datasets

\section{Latent Variable Model}
Our method retains the same latent variable structure from PLSI \cite{PLSI}. However, instead of using EM algorithm, we extract the parameters by factorizing second and third order moments of the dataset.

\subsection{Generative Model}

Let us assume that there are $N$ users and $D$ items, and the latent variable $h$ can assume $K$ states. For any user $u \in \{ u_1,u_2 \dots u_N \} $, if $n_u$ is the number of items associated with $n_u$, then we first choose a latent state of $h \in \{1,2 \dots K\}$ from the discrete distribution $P\big[h|u\big]$, then we choose an item $y \in \{ y_1,y_2 \dots y_D \}$ from the discrete distribution $P\big[ y|h \big]$, and repeat it for $n_u$ times. The final sample $x_u \in \mathbb{R}^D$ for user $u$ is a binary vector with an entry $1$ for items sampled for user $u$, and $0$ elsewhere, resulting in $|x_u| = n_u$.

The generative process is as follows.

\begin{gather*} 
\text{For every user } u \in \{ u_1 \dots u_N\} \text{, repeat for } n_u \text{ times: } \\
h \sim Discrete(P\big[ h|u \big]) \\
 y \sim Discrete(P\big[ y|h \big]) \\
\numberthis
\label{1}
\end{gather*}

Let us denote the probability of the latent variable $h$ assuming the state $k \in {1 \dots K} $ as,

\begin{equation}
\pi_k = P\big[ h=k \big]
\end{equation}

Let us define $\bar{\mu}_k \in \mathbb{R}^D  $ as the probability vector of all the items conditional to the latent state $k \in {1 \dots K} $, i.e. 
\begin{equation}
\bar{\mu}_k=P\big[ y|h=k \big]
\end{equation}

Let the matrix $O \in \mathbb{R}^{D \times K}$ denote the conditional probabilities for the items, i.e. $O_{i,k}=P\big[ y_i|h=k \big]$.
Then $O=[\bar{\mu}_1|\bar{\mu}_2| \dots |\bar{\mu}_K] $. We assume that the matrix $O$ is of full rank, and the columns of $O$ are fully identifiable. The aim of our algorithm is to estimate the matrix $O$ as well as the vector $\pi$, and then derive the user personalization parameters $P \big[ h=k|u \big]$ from them.

Following the generative model in equation \ref{1}, we can define the probability of individual item as,
\begin{align*}
P[y_j]=\sum_{k=1}^K{P[y_j|h]P[h=k]}=\sum_{k=1}^K[\bar{\mu}_j]\pi_k, \hfill \forall j = 1,2, \dots D
\end{align*}

Therefore, the average probability of the items across the data can be defined as,
\begin{align*}
M_1&={P[y_1,y_2,\dots y_D]}^\top \\
& =\sum_{k=1}^K{\pi_k \left[ [\bar{\mu}_k]_1, [\bar{\mu}_k]_1 \dots [\bar{\mu}_k]_D \right]}^\top \\
&=\sum_{k=1}^K{\pi_k \bar{\mu}_k}
\numberthis
\end{align*}

Now, we try to formulate the matrix of the pairwise probability of the items. Let us assume that we choose two items $w_1$ and $w_2$ from the list of any user at random. The probability $P[w_1=y_i]$ represents the probability by which any item picked at random from the item lists of the users turns out to be $y_i$, and it is nothing but $P[y_i]$.

Similarly, $P[w_1=y_i,w_2=y_j]$ represents the probability by which two items picked at random from the item lists turn out to be $y_i$ and $y_j$, and it is same as $P[y_i,y_j]$, with $i, j = 1,2 \dots D$.
Now, from the generative process in Equation \ref{1}, $w_1$ and $w_2$ are conditionally independent given $h$, i.e., $P[w_1,w_2|h=k] = P[w_1|h=k]P[w_2|h=k]$ with $k=1,2 \dots K$.
Therefore,
\begin{align*}
&P[y_i,y_j]\\
&=P[w_1=y_i,w_2= y_j] \\
&=\sum_{k=1}^K{P[w_1=y_i,w_2= y_j|h=k]P[h=k]} \\
&=\sum_{k=1}^K{P[w_1=y_i|h=k]P[w_2=y_j|h=k]P[h=k]} \\
&=\sum_{k=1}^K{P[y_i|h=k]P[y_j|h=k]P[h=k]} \\
&=\sum_{k=1}^K{[\bar{\mu}_k]_i [\bar{\mu}_k]_j \pi_k}, \forall i,j \in \{ 1,2 \dots D\}
\end{align*}

Defining $M_2$ as the pairwise probability matrix, with $[M_2]_{i,j}=P\big[y_i,y_j\big]$, we can express it as,
\begin{equation}
M_2=\sum_{k=1}^K{\pi_k\bar{\mu}_k {\bar{\mu}_k}^\top } =\sum_{k=1}^K{\pi_k\bar{\mu}_k \otimes \bar{\mu}_k}
\end{equation}

Similarly, the tensor $M_3$ defined as the third order probability moment, with $[M_3]_{i,j,l} =P[y_i,y_j,y_l]$, can be represented as,

\begin{equation}
M_3=\sum_{k=1}^K{\pi_k\bar{\mu}_k \otimes \bar{\mu}_k \otimes \bar{\mu}_k}
\end{equation}

\subsection{Parameter Extraction}

The first step of parameter extraction is to whiten the matrix $M_2$, where we try to find a low-rank matrix $W$ such that $W^\top M_2 W=I$. This is a method similar to the whitening in ICA, with the covariance matrix replaced with the co-occurrence probability matrix in our case.

The whitening is usually done through eigenvalue decomposition of $M_2$. If the $K$ maximum eigenvalues of $M_2$ are $\{ \nu_k \}_{k=1}^K$, and the corresponding eigenvectors are $\{ \omega_k \}_{k=1}^K$, then the whitening matrix of rank $K$ is computed as $W=\Omega{\Sigma}^{-1/2}$, where $\Omega=\big[ \omega_1|\omega_2| \dots |\omega_K \big]$, 
\& $\Sigma = diag(\nu_1,\dots,\nu_K)$. 

Upon whitening $M_2$ takes the form 
\begin{align*}
W^\top M_2W &= W^\top \big( \sum_{k=1}^K{\pi_k\bar{\mu}_k  \bar{\mu}_k^\top} \big)W\\
& = \sum_{k=1}^K{ \big( \sqrt{\pi_k}W^\top \bar{\mu}_k \big)  \big( \sqrt{\pi_k}W^\top \bar{\mu}_k \big)^\top  } \\
& = \sum_{k=1}^K \tilde{\mu}_k \tilde{\mu}_k^\top  = I 
\numberthis \label{eqn:whiten}
\end{align*}

Hence $ \tilde{\mu}_k = \sqrt{\pi_k}W^\top\bar{\mu}_k $ are orthonormal vectors. 
Multiplying $M_3$ along all three dimensions by $W$, we get
\begin{align*}
 \tilde{M_3} &= M_3(W,W,W) \\
 & = \sum_{k=1}^K \pi_k(W^\top \bar{\mu}_k) \otimes (W^\top \bar{\mu}_k) \otimes (W^\top \bar{\mu}_k) \\
 & = \sum_{k=1}^K \frac{1}{\sqrt{\pi_k}} \tilde{\mu}_k \otimes  \tilde{\mu}_k \otimes \tilde{\mu}_k
\numberthis
\end{align*}

Upon canonical decomposition of  $\tilde{M_3}$, if the eigenvalues and eigenvectors are  $\{ \lambda_k\}_{k=1}^K $ and $\{v_k\}_{k=1}^K$ respectively, then $\lambda_k = \sfrac{1}{\sqrt{\pi_k}} $. i.e., $\pi_k = \lambda_k^{-2}$, and,
\begin{equation}
v_k = \tilde{\mu}_k = \sqrt{\pi_k}W^\top\bar{\mu}_k= \frac{1}{ \lambda_k}W^\top\bar{\mu}_k
\end{equation}

The $\bar{\mu}_k $s can be recovered as $\bar{\mu}_k = \lambda_kW^\dagger v_k $, where $W^\dagger$ is the pseudo-inverse of $W^\top$, i.e., $W^\dagger =  W \left(W^\top W\right)^{-1} $. Since we normalize the columns of $O$ as $O_{yk}=\frac{O_{yk}}{\sum_v{O_{yk}}}$. it is sufficient to compute $\bar{\mu}_k = W^\dagger u_k$, since $\lambda_k$ will be cancelled during normalization. The matrix $O$ can be constructed as $O=\big[\bar{\mu}_1|\bar{\mu}_2|\dots|\bar{\mu}_K \big]$.

\subsection{User Personalization}

Once we have $O$ and $\pi$, the probability of a user $u \in \{ u_1,u_2 \dots u_N \} $ given $h$ can be expressed as,

\begin{equation}
P\big[u|h=k\big]   = \prod_{y \in Y_u} P\big[y|h=k\big]
\end{equation}
where $Y_u$ is the list of items selected by the user $u$ in the training set. 

Then the user personalization probabilities $P\big[ h=k|u \big]$ can be estimated using Bayes Rule. 
\begin{align*}
 P\big[h=k|u\big]&=\frac{P\big[h=k\big]\prod_{y \in Y_u}P\big[ y|h=k \big]}{\sum_{k=1}^K P\big[h=k\big] \prod_{y \in Y_u}P\big[ y|h=k\big]} \\
 &=\frac{\pi_k \prod_{y \in Y_u}O_{yk}}{\sum_{k=1}^K \pi_k\prod_{y \in Y_u}O_{yk}}
  \numberthis \label{eqn:personalization}
\end{align*}

Finally, we compute the probability of a user $\tilde{u}$ selecting an item $\tilde{y}$ by the following equation, and recommend the items with the highest probability for the user $\tilde{u}$.

\begin{align*}
\label{eqn:prediction}
P\big[ \tilde{y}|\tilde{u}\big] &= \sum_{k=1}^KP\big[\tilde{y}|h=k\big] P\big[h=k|\tilde{u}\big] \\
&= \frac{\sum_{k=1}^K\pi_k O_{\tilde{y}k} \prod_{y \in Y_{\tilde{u}} }O_{yk}}{\sum_{k=1}^K \pi_k\prod_{y \in Y_{\tilde{u}} }O_{yk}}
\numberthis
\end{align*}

Please note that although we use the same latent variable model as PLSI \cite{PLSI}, our model parameters are only $O$ and $\pi$. Therefore our number of effective parameters is only $(D-1)K + (K-1)$, unlike the case of PLSI that uses $(D-1)K + N(K-1)$ parameters.The personalization parameters are not model parameters in our case since we derive them from  $O$ and $\pi$.

\begin{algorithm*}[tb]
\addtolength\linewidth{-9ex}
   \caption{Method of Moments for Parameter Extraction}
   \label{alg:mom}
\begin{algorithmic}
   \STATE {\bfseries Input:} Sparse Data $ X \in \mathbb{R}^{N \times D}$ and $K \in \mathbb{Z}^+$ \\{\bfseries 
   Output:} $ P\big[y|h\big]$ and $ P\big[h|u\big] $

\begin{enumerate}
\item Estimate $\hat{M}_2 = (X^\top X) /  \sum_{i=1}^N{nnz(x_i)^2}$ \hfill ($\#$ of pass 1)
  \item Compute $K$  maximum eigenvalues of $\hat{M}_2$ as $\{ \nu_k \}_{k=1}^K$, and corresponding eigenvectors as  $\{ \omega_k \}_{k=1}^K$. Define $\Omega=\big[ \omega_1|\omega_2| \dots |\omega_K \big]$, and $\Sigma = diag\left(\nu_1,\nu_2,\dots,\nu_K\right)$
  \item Estimate the whitening matrix $ \hat{W} = \Omega{\Sigma}^{-1/2}$ so that $ \hat{W}^\top \hat{M}_2 \hat{W} =I_{K\times K}$
  \item Estimate $\hat{\tilde{M}}_3=(X\hat{W} \otimes X\hat{W}\otimes X\hat{W}) / \sum_{i=1}^N{nnz(x_i)^3}  $  \hfill ($\#$ of pass 2)
  \item Compute eigenvalues $\{ \lambda_k\}_{k=1}^K $ and eigenvectors $\{v_k\}_{k=1}^K $ of $ \hat{\tilde{M}}_3$ 
  \item Estimate the columns of $O$ as $\hat{\bar{\mu}}_k = \hat{W}^\dagger v_k $, where 
  $\hat{W}^\dagger =  \hat{W}{(\hat{W}^\top \hat{W})}^{-1} $ , and $\hat{\pi}_k= \lambda_k^{-2}$, 
  $\forall k \in 1,2 \dots K$ 
  \item Assign $\hat{O} = [\hat{\bar{\mu}}_1|\hat{\bar{\mu}}_2|\dots|\hat{\bar{\mu}}_K] $ \&
  $\hat{\pi}=[\hat{\pi}_1, \hat{\pi}_2 \dots \hat{\pi}_K]^\top $
  \item Estimate $P\big[ y|h=k \big]= \frac{\hat{O}_{yk}}{\sum_{y} {\hat{O}_{yk}}}, \forall k \in 1\dots K, y \in y_1\dots y_D, \newline P\big[ h=k|u \big]  =\frac{\hat{\pi}_k\prod_{y \in Y_u}\hat{O}_{yk}}{\sum_{k=1}^K \hat{\pi}_k\prod_{y \in Y_u}\hat{O}_{yk}}  \forall k \in 1\dots K, u \in u_1\dots u_N $  \hfill ($\#$ of pass 3)
\end{enumerate}
\end{algorithmic}
\end{algorithm*}

\section{Implementation Detail}

We create an estimation of the sparse moments $M_2$ by counting the pairwise occurrence of the items across the selections made by all the users in the dataset, and normalizing by the total number of occurrence in each case. This can be achieved in one pass through the dataset using frameworks like Hadoop. Alternately, if $X \in \mathbb{R}^{N \times  D}$ is the binary sparse matrix representing the data, then the pairwise occurrence matrix can be estimated by $X^\top X$, whose sum of all elements is,
\begin{align*}
\sum_y \sum_y X^\top X &= \sum_y \sum_y \sum_{i=1}^N x_i^\top x_i \\
 &=\sum_{i=1}^N  \sum_y \sum_y x_i^\top x_i \\
  &=\sum_{i=1}^N nnz(x_i)^2 \\
\end{align*}
where $x_i$ is the row of $X$ corresponding to the $i$th user, and $nnz(x_i)$ is the number of non-zero elements in $x_i$, i.e., the number of items associated with $i$th user. Therefore, $M_2$ can be estimated as,

\begin{equation}
\hat{M}_2 = \frac{1}{\sum_{i=1}^N{nnz(x_i)^2}}X^\top X
\end{equation}

Similarly, the triple-wise occurrence tensor can be estimated as $X \otimes X \otimes X$, and the sum of all of the elements of the tensor is $\sum_y \sum_y \sum_y X \otimes X \otimes X = \sum_{i=1}^N nnz(x_i)^3$. Therefore, $M_3$ can be estimated as,
\begin{equation}
\hat{M}_3 = \frac{1}{\sum_{i=1}^N{nnz(x_i)^3}}X \otimes X \otimes X
\end{equation}

The dimensions of $\hat{M}_2$ and $\hat{M}_3$ are $D^2$ and $D^3$ respectively, but in practice, these quantities are extremely sparse. Also, we can estimate $\tilde{M_3}$ without estimating $M_3$. Since $\hat{\tilde{M}}_3=\hat{M}_3(W,W,W)$, it can be estimated as,
\begin{equation}
\label{eqn:tM3}
\hat{\tilde{M}}_3=\frac{1}{\sum_{i=1}^N{nnz(x_i)^3}}XW \otimes XW \otimes XW
 \end{equation}
 
 $\tilde{M}_3$ has a dimension of $K^3$, and can be conveniently stored in the memory ($K\ll D$). Estimating $\tilde{M_3}$ takes a second pass through the entire dataset. We used the Tensor Toolbox \cite{TTB_Software} for tensor decomposition. Once the matrix $O$  and  $\pi_k$ are extracted, it requires one more pass through the entire dataset to compute the user probabilities ($P[h|u]$), resulting a total of three passes for the extraction of all parameters. The entire algorithm is outlined as Algorithm \ref{alg:mom}. Although it is possible to make predictions using only $O$ and $\pi$, it is advisable to compute $P[h|u]$ beforehand to avoid computation cost during prediction step.

The number of elements in $M_2$ is $\mathcal{O}\left( \sum_{i=1}^N nnz(x_i)^2 \right)$, with the worst case occurring when no two users has any item in common, and all the elements in $X^\top X$ is one. The complexity of extracting $K$ largest eigenvalue of $M_2$ during the whitening step is $\mathcal{O}\left( K \left( \sum_{i=1}^N nnz(x_i)^2 \right) \right)$. The complexity of Equation \ref{eqn:tM3} is $\Theta(NK^3)$. The tensor factorization step has a complexity of $\mathcal{O}\left( K^4 \log{(1/\epsilon)} \right)$ to extract all $K$ eigenvalues of $\tilde{M}_3$ up to an accuracy of $\epsilon$. These three steps contribute the most to the computational burden of the algorithm. The complexity of the overall parameter extraction stage is,
\begin{equation*}
 \mathcal{O}\Big( K \Big( \sum_{i=1}^N nnz(x_i)^2 \Big) + NK^3 +K^4 \log{(1/\epsilon)}  \Big)
 \end{equation*}

\newpage
\subsection{Convergence Bound}

\newtheorem{theorem}{Theorem}
\label{thm:bound}
\begin{theorem}

Let us assume that we draw $N$ i.i.d samples $x_1, x_2 \dots x_N$ corresponding to $N$ users using the generative process in Equation \ref{1}. Let us define $\varepsilon = \left( 1 + \sqrt{\frac{\log (1 / \delta)}{2}}\right)$ for some $\delta \in (0,1)$. Then, if the number of users $N \ge \max ( n_1, n_2, n_3)$, where

\begin{itemize}
\item  $n_1 = c_2\left( \log{K} + \log{\log{\left(\frac{K}{c_1}   \cdot \sqrt{\frac{\pi_{max}}{\pi_{min}}} \right)}} \right)$
\item  $n_2 =  \Omega\left( \left( \frac{\varepsilon}{\tilde{d}_{2s}\sigma_K(M_2)} \right)^2 \right)$
\item  $n_3 = \Omega\left( K^2 \left( \frac{10}{\tilde{d}_{2s}\sigma_K(M_2)^{5/2}}  + \frac{2\sqrt{2}}{\tilde{d}_{3s}\sigma_K(M_2)^{3/2}} \right)^2 \varepsilon^2 \right)$
\end{itemize}

for some constants $c_1$ and $c_2$, and we run Algorithm \ref{alg:mom} on the data, then the following bounds on the estimated parameters hold with probability at least $1 - \delta$, 
 
 \begin{align*}
 ||\mu_k-\hat{\mu}_k||  \le   \left( \frac{160\sqrt{\sigma_1(M_2)}}{\tilde{d}_{2s}\sigma_K(M_2)^{5/2}}  + \frac{32\sqrt{2\sigma_1(M_2)}}{\tilde{d}_{3s}\sigma_K(M_2)^{3/2}} \right. +\\
 \left. \frac{4 \sqrt{\sigma_1(M_2)} }{\tilde{d}_{2s}\sigma_K \left( M_2 \right)} \right) \frac{\varepsilon}{\sqrt{N}}
 \end{align*}

 and,
 \begin{align*}
  |\pi_k-\hat{\pi}_k| \le \left( \frac{200}{\sigma_K(M_2)^{5/2}}+ \frac{40\sqrt{2}}{\sigma_K(M_2)^{3/2}} \right)\frac{\varepsilon}{\tilde{d}_{3s}\sqrt{N}}
  \end{align*}
  where $\sigma_1(M_2) \dots \sigma_K(M_2)$ are the K largest eigenvalues of the pairwise probability matrix $M_2$, $\tilde{d}_{2s} = \frac{1}{N} \sum_{i=1}^N nnz(x_i)^2$ and  $\tilde{d}_{3s} = \frac{1}N \sum_{i=1}^N nnz(x_i)^3$, with $nnz(x_i)$ representing the non-zero elements in the $i$th sample.

\end{theorem}

The proof is in the appendix.

\begin{table*}[tb]
\centering
\caption{Descriptions of the Datasets}
\label{table:description}
\begin{tabular}{|c|c|c|c|p{1.65cm}|p{1.6cm}|p{1.65cm}|} \hline
Name & Type  & $\#$ of Users & $\#$ of Items & $\#$ of Tuples (training) & Sparsity (training) & $\#$ of Tuples (test)\\ \hline
 Ta-Feng &  Online Grocery  & $24,304$ & $21,533$ & $ 417,246 $ & $5.44\times10^{-4}$ & $274,479$\\ \hline
 Million Song & Music Subscription & $110,000$ & $163,206$ & $1,450,933$ & $8.08\times10^{-5}$ & $1,368,430$\\ \hline
 Yandex & Search Engine logs & $1 M$ & $718,675$ & $5,669,541 $ & $7.89\times10^{-6}$ & $3,516,216$\\ \hline
\end{tabular}
\end{table*}

\begin{figure*}[t]
\centering
\subfigure[Precision-Recall (TaFeng)]{\includegraphics[scale=0.27]{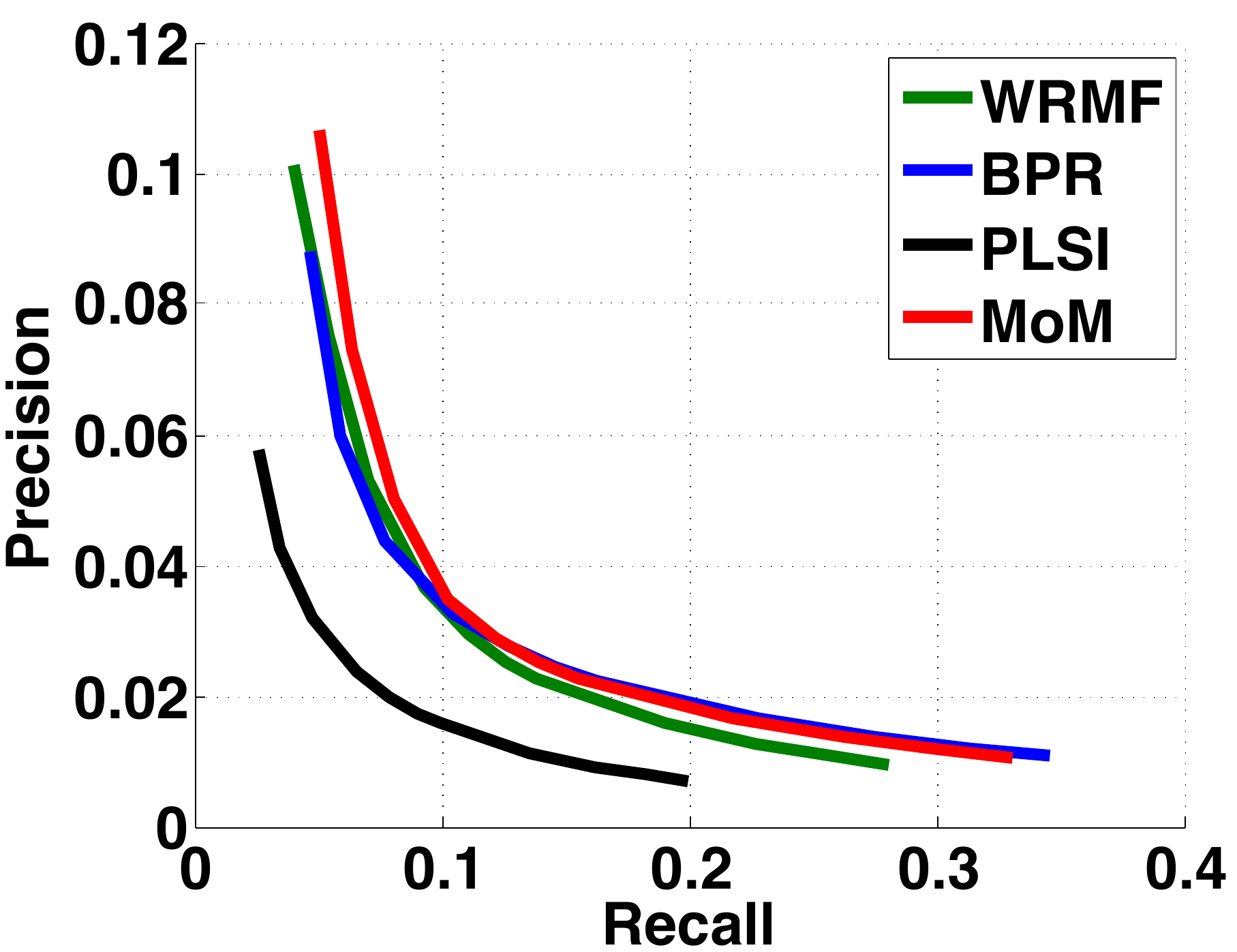}}
      ~
\subfigure[Precision-Recall (Million Song)]{\includegraphics[scale=0.27]{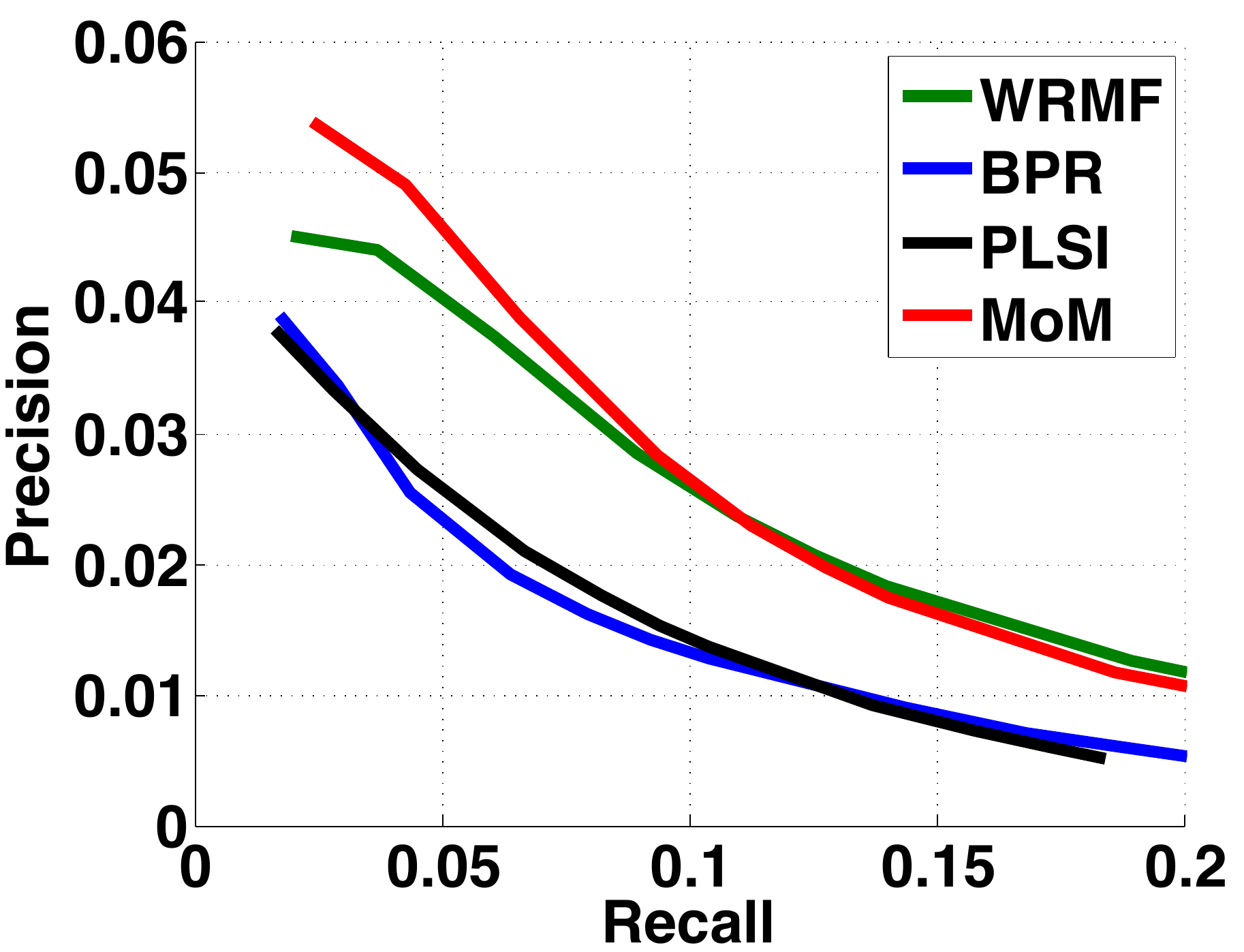}}
      ~
\subfigure[Precision-Recall (Yandex)]{\includegraphics[scale=0.27]{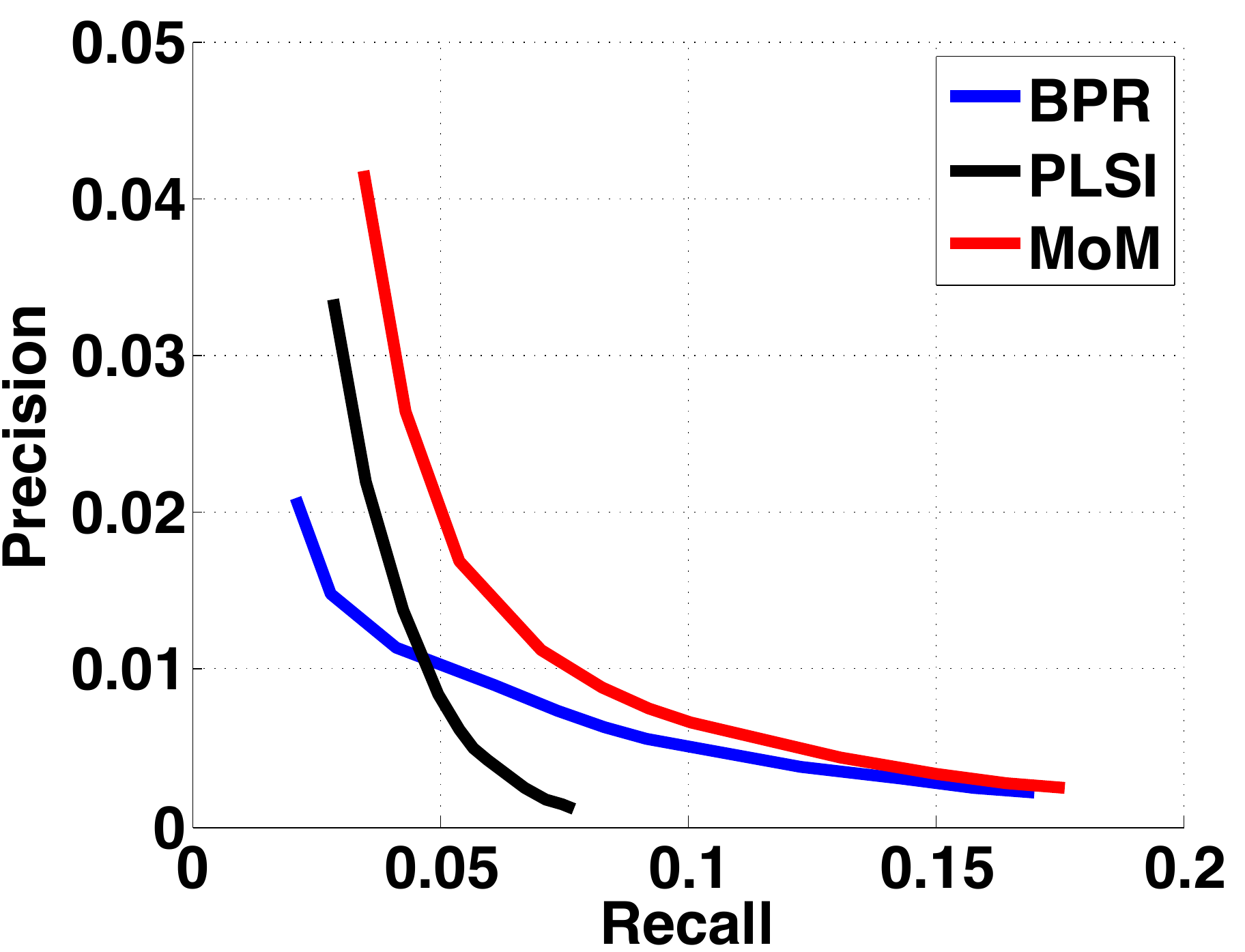}}

\subfigure[MAP (Ta-Feng)]{\includegraphics[scale=0.27]{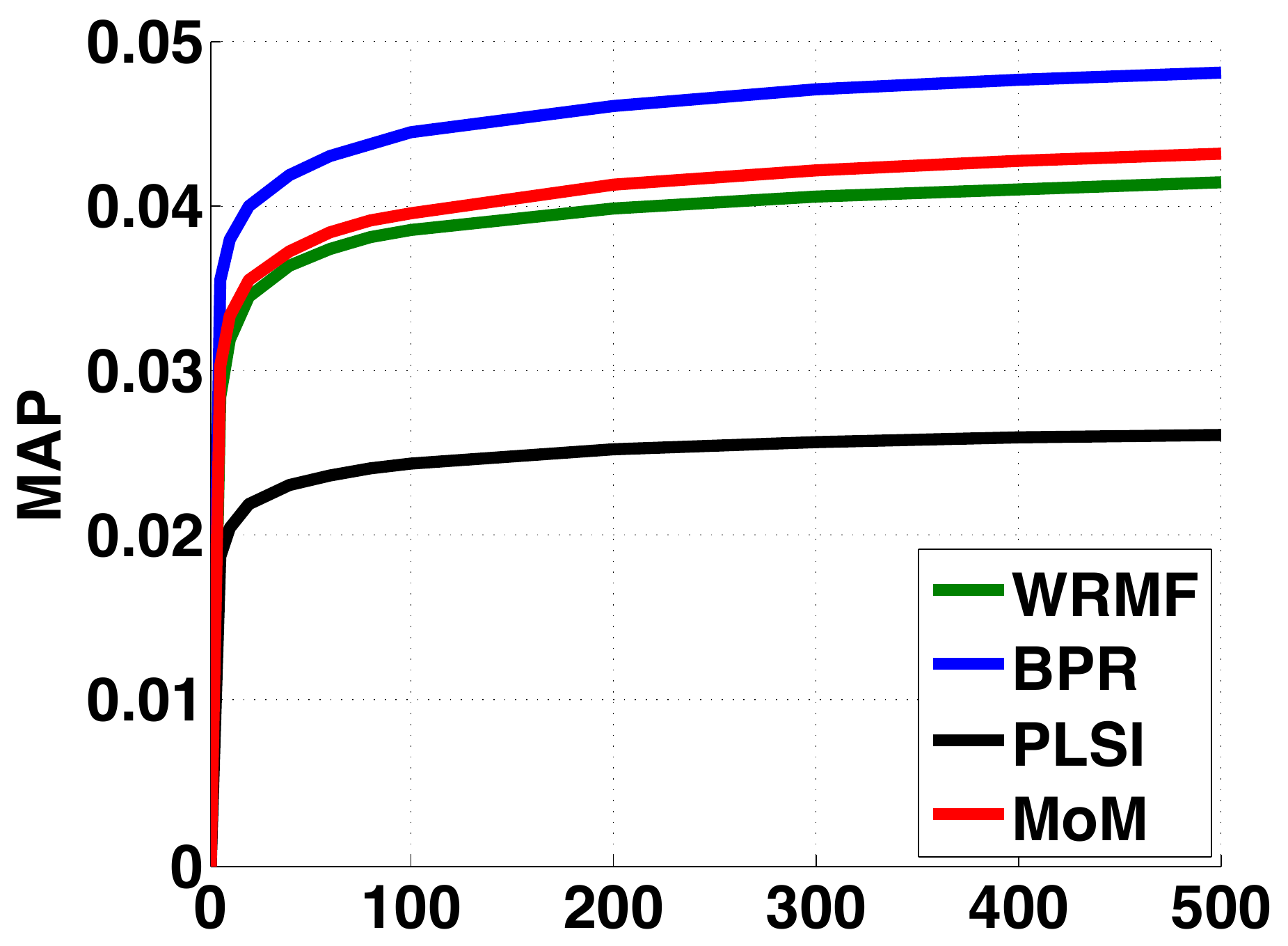}}
      ~
\subfigure[MAP (Million Song)]{\includegraphics[scale=0.27]{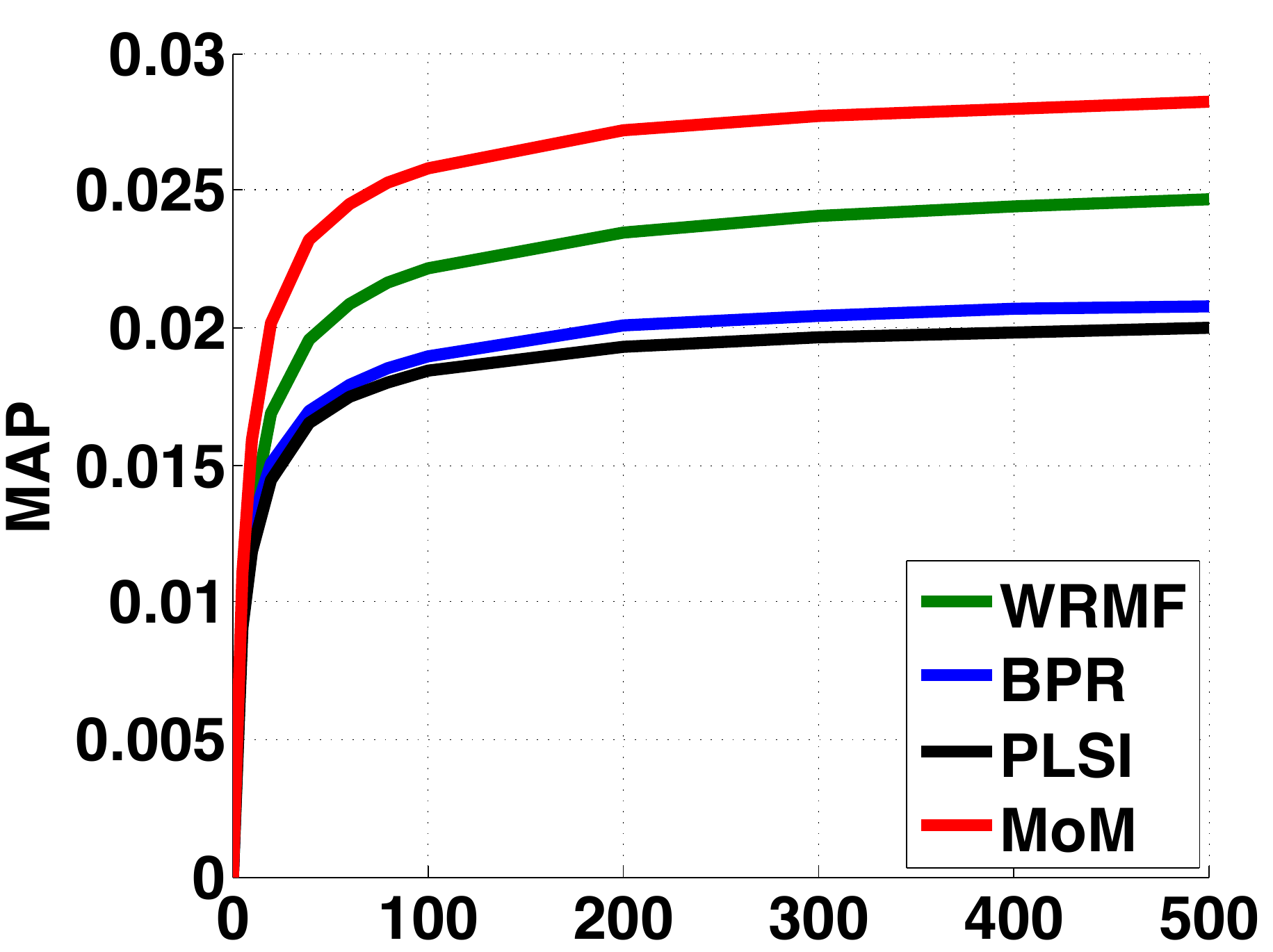}}
      ~
\subfigure[MAP (Yandex)]{\includegraphics[scale=0.27]{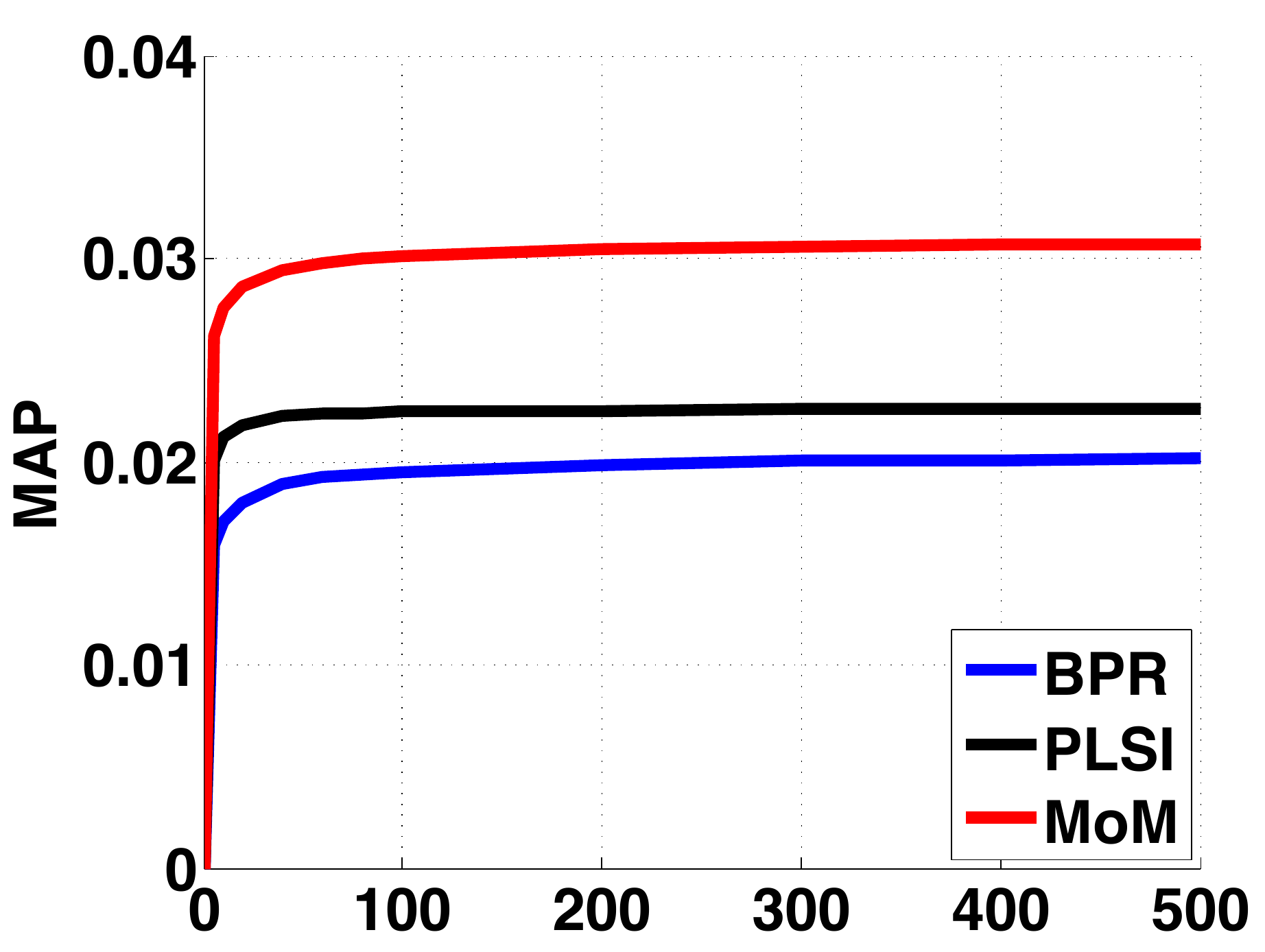}}

\caption{Precision-Recall curves (1st row) and Mean Average Precision or MAP (2nd row) of different methods on the three datasets }
\label{fig:sim}
\end{figure*}

\section{Experimental Results}

We show the implementation of our model on three publicly available datasets so that the results can be reproduced whenever necessary. The datasets contain records of user-item interactions over a period, and truly represents implicit feedback systems. We do not convert any dataset with user ratings into implicit feedback dataset, as it may not be an accurate representation of implicit feedback scenario.

The different attributes of datasets are described in Table \ref{table:description}. We use $K=100$ for all the models in our experiments. For the standard form of PLSI, we run EM algorithm until $L^t - L^{t-1} < .001\times |L^{t-1}|$, where $L^t$ is the log-likelihood at iteration $t$, resulting in around $25-30$ iterations for each dataset. We use the implementation of WRMF and BPR from MyMediaLite library \footnote{\url{http://www.mymedialite.net/}} developed by the authors of \cite{BPR}. We found that the rest of the algorithms, such as \cite{MATCOMP}, \cite{GBPR} or \cite{LORSLIM} lacks scalability to train on large datasets, at least in their current implementation provided by the authors. We could not find an implementation of AdaBPR\cite{AdaBPR} from the authors. The article uses much smaller datasets, e.g. the authors select only 27,216 users and 9,994 songs from the Million Song dataset.  We used WRMF and BPR for the benchmarking purpose since most of the relevant literature on recommendation systems considered these two algorithms as the state-of-art.

For every dataset, we compute the Precision@$\tau$, Recall@$\tau$, and Mean Average Precision (MAP@$\tau$) for $\tau \in \{5,10, 20, 40,$ $ 60, 80, 100, 200, 300, 400, 500\}$. The Precision-Recall curves as well as MAP@$\tau$ is shown in Figure \ref{fig:sim}, and the computation time in Table \ref{table:AUC}. We carried out our experiments on Unix Platform on a single machine with a single-core 2.4GHz processor and 8GB memory, and did not use multi-threading or any other performance enhancement technique. \footnote{The code and the data will be released upon acceptance}

\begin{table}[tb]
\caption{ Computation Time (sec) }
\centering
\label{table:AUC}
\begin{tabular}{ccccc} \toprule
 Dataset  & WRMF & BPR & PLSI & MoM\\ \midrule
 Ta-Feng & 2200 & 108 & 1081 & 120\\ \midrule
 Million Song & 14262 & 510 & 3036 & 600 \\ \midrule
 Yandex          &- & 2100 & 15300 & 2512 \\ \hline
\end{tabular}
\end{table}

\subsection{Ta-Feng Dataset}
Ta-Feng dataset consists of online grocery purchase records for the months of January, February, November and December in 2001.We combine the records of January and November resulting in a training set consisting of around 24,000 users and 21,000 products, and around 470,000 sales records. The records of February and December are combined to form the test set. BPR achieves the highest MAP of all, but MoM produces the best Precision-Recall curve, taking similar time as BPR.

\subsection{Million Song Dataset} 
Million Song dataset contains the logs of 1 million users listening to 385,000 song tracks with 48 million observations. Here, we use a subset of the data consisting of 100,000 users and around 165,000 song tracks with around 1.45 million observations released in Kaggle. MoM performs the best regarding MAP and Precision-Recall, except for higher values of $\tau$ when WRMF catches up.

\subsection{Yandex Search Log Dataset}
Yandex dataset contains the search logs of 27 days for 5.7 million users and 70.3 million URLs. We selected 718,675 URLs, each of which had at least five clicks since it is not possible to personalize URLs with very few clicks. We randomly selected 1M users who clicked one of those 718,675 URLs.We used the data of first 14 days as the training set, and the last 13 days as the test set. WRMF did not finish even after running for a day.  MoM outperformed BPR and PLSI while taking similar time as BPR.

\section{Conclusion}

Here we propose a collaborative filtering algorithm for implicit feedback based on the second and third order moment factorization of the data. Existing methods like PLSI suffers from local maxima problem. Although Matrix factorizations operate on a convex loss, it is far from trivial to reach the global minima of the loss function through gradient descent alternately on user and item features. The Method of Moments, on the other hand, comes with guaranteed convergence bound. The only drawback of Method of Moments is that it will not work when there are only a few users available such that $N < \Theta(K^2)$. However, modern recommendation systems usually operate on a large number of users, and this is far from a possibility.

We demonstrate the competitive performance of Method of Moments through experiments on three real-world datasets, chosen from different domains. BPR performs better in MAP for Ta-Feng datasets. However, as the size and the sparsity of the datasets increase, the performance of BPR gets worse. Method of Moments performs the best for Million Song and Yandex datasets while taking similar time as BPR. PLSI or Matrix Factorization (WRMF) clearly lacks the scalability that MoM offers, neither do they produce any better result. Further, MoM depends only on various linear algebraic operations, and it is embarrassingly parallel to implement on any parallel platforms. This makes it a very suitable choice for recommendation from large-scale Implicit Feedback datasets.

\nocite{UBM} \nocite{WBPR} \nocite{TTB_Software} \nocite{TTB_SSHOPM} \nocite{CollapsedGibbs} \nocite{LDA} \nocite{content} \nocite{McFee:2012:MSD:2187980.2188222} \nocite{Weston2010}
\nocite{AUC} \nocite{MATCOMP}
\bibliographystyle{abbrv}
\bibliography{ICLR_RecSys}  

\appendix
\section{Vector Norms} 

Let the true pairwise probability matrix and the third order probability moment be $M_2=p(y,y)$ and $M_3=p(y,y,y)$, where $y$ stands for the items.  Let us assume that we select $N$ i.i.d. samples $x_1, \dots x_N$ from the population, and the estimates of pairwise matrix and third order moment are $\hat{M}_2 = \hat{p}(y,y)$ and $\hat{M}_3 = \hat{p}(y,y,y)$. Let $\varepsilon_{M_2} = ||M_2 - \hat{M}_2||_2$. We use the second order operator norm of the matrices here. Let us assume $\varepsilon_{M_2} \le \sigma_K(M_2)/2$, where $\sigma_K$ is the $K$th largest eigenvalue of $M_2$. We will derive the conditions which satisfies this later. 

If $\Sigma=diag(\sigma_1,\sigma_2 \dots \sigma_K)$ are the top-K eigenvalues of $M_2$, and $U$ are the corresponding eigenvectors, then the whitening matrix $W=U\Sigma^{-1/2}$, and, $W^\top M_{2} W = I_{K \times K}$. Therefore,

 \begin{align*}
 ||W||_2 &= \sqrt{ \max \text{eig}(W ^\top W)}=\sqrt{ \max \text{eig}(\Sigma^{-1})} =\frac{1}{\sqrt{\sigma_K(M_2)}}
 \end{align*}
 
 Similarly, if $W ^\dagger = W(W^\top W)^{-1}$, then $W ^\dagger = W\Sigma=U\Sigma^{1/2}$. Therefore,
 
 \begin{equation}
 ||W ^\dagger||_2=\sqrt{ \max \text{eig}(\Sigma)}=\sqrt{\sigma_1 (M_2)}
 \end{equation}

Let $\hat{W}$ be the whitening matrix for $\hat{M}_2$, i.e., $\hat{W}^\top \hat{M}_{2} \hat{W} = I_{K \times K}$. Then by Weyl's inequality, \\
$\sigma_k(M_2) - \sigma_k(\hat{M}_2) \le ||M_2 - \hat{M}_2||, \forall k = 1,2 \dots K$. 

Therefore,
\begin{align*}
||\hat{W}||_2^2 &=\frac{1}{\sigma_K (\hat{M}_2)} \\
& \le \frac{1}{\sigma_K\left( M_2 \right) - ||M_2 - \hat{M}_2||} \\
&\le \frac{2}{\sigma_K\left( M_2 \right)} \numberthis
\end{align*}

Also, by Weyl's Theorem,
\begin{align*}  
& ||\hat{W} ^\dagger||_2^2 = \sigma_1(\hat{M}_2) \le   \sigma_1(M_2) + \varepsilon_{M_2} \le 1.5 \sigma_1(M_2)  \\
& \implies ||\hat{W} ^\dagger||_2  \le \sqrt{1.5\sigma_1(M_2)}  \le 1.5\sqrt{\sigma_1(M_2)} \numberthis
\end{align*}

Let $D$ be the eigenvectors of $\hat{W}M_2 \hat{W}$, and $A$ be the corresponding eigenvalues. Then we can write, $\hat{W}M_2 \hat{W}$=$ADA^\top$.  Then $W = \hat{W}A D^{-1/2} A^\top$ whitens $M_2$, i.e., $W^\top M_2 W = I$. Therefore,

\begin{align*}
||I-D||_2 &= ||I - ADA^\top||_2\\
&=||I-\hat{W}M_2\hat{W}||_2\\
& =|| \hat{W}\hat{M}_2\hat{W}-\hat{W}M_2\hat{W}||_2  \\
& \le ||\hat{W}||_2^2 ||M_2 - \hat{M_2}|| \\
&\le \frac{2}{\sigma_K \left( M_2 \right)} \varepsilon_{M_2} \numberthis
\end{align*}

\begin{align*}
  \varepsilon_{W} &= ||W-WA D^{1/2} A^\top||_2 \\
 &= ||W||_2||I-A D^{1/2} A^\top||_2 \\
 &=||W||_2|| I-D^{1/2}||_2\\
& \le ||W||_2 || I-D^{1/2}||_2 || I + D^{1/2}||_2 \\
&= ||W||_2 || I-D||_2 \\
& \le \frac{2}{\sigma_K(M_2)^{3/2}} \varepsilon_{M2}
\end{align*}

\begin{align*}
\varepsilon_{W^\dagger} &= ||{W}^\dagger - \hat{W}^\dagger||_2 \\
&= ||\hat{W}^\dagger A D^{1/2}A^\top - \hat{W}^\dagger||_2 \\
& = ||\hat{W}^\dagger||_2 || I - A D^{1/2}A^\top||_2 \\
& \le ||\hat{W}^\dagger||_2 || I - D||_2 \le \frac{2 \sqrt{\sigma_1(M_2)} }{\sigma_K \left( M_2 \right)} \varepsilon_{M_2} \numberthis
\end{align*}

\section{Tensor Norm}
Let us define the second order operator norm of a tensor $T \in \mathbb{R}^{D \times D \times D}$ as,
\begin{equation}
||T||_2 = \sup_v \{ |T(v,v,v)| : v \in \mathbb{R}^D \& ||v||=1  \}
\end{equation}

\newtheorem{lemma}{Lemma}
\begin{lemma} 
\label{lemma:tnorm}
For a tensor $T \in \mathbb{R}^{D \times D \times D}$, $||T||_2 \le ||T||_F$, where $||T||_F$ is the Frobenius norm defined as, 
\begin{equation*}
||T||_F = \sqrt{\sum_{i,j,k} (T_{i,j,k})^2}
\end{equation*}
\end{lemma}

\begin{proof}
For any real matrix $A$, $||A||_2 \le ||A||_F$. Let us unfold the tensor $T$ as the collection of $D$ matrices, as, $T = \{ T_1,T_2 \dots T_D \}$.
Then,

\begin{align*}
T(v,v,v) &= v^\top[T_1v | T_2v |\dots |T_Kv]v \\
& =\langle [v^\top T_1v, v^\top T_2v, \dots v^\top T_Kv], v \rangle \numberthis
\end{align*}

Therefore,
\begin{align*}
||T||_2 &= \sup_v \{ |T(v,v,v)| : v \in \mathbb{R}^D \& ||v||=1  \}  \\
&= \sup_v\{ \left|  \langle [v^\top T_1v , v^\top T_2v , \dots , v^\top T_Kv],v \rangle  \right|  : v \in \mathbb{R}^D \\
& \phantom{ \sup_v\{ \langle [v^\top T_1v,  v^\top T_2v |\dots, v^\top T_Kv],v \rangle} \& ||v||=1  \}  \\ 
\end{align*}

Using Cauchy-Schwarz inequality,

\begin{align*}
||T||_2 & \le \sup_v\{  \left|\left| [v^\top T_1v, v^\top T_2v, \dots v^\top T_Kv] \right| \right| ||v||   \\
&\phantom{ \sup\{  \left|\left| [v^\top T_1v | v^\top T_2v |\dots ] \right| \right| } : v \in \mathbb{R}^D  \& ||v||=1  \\
&=  \sup_v\{  \left|\left| [v^\top T_1v , v^\top T_2v , \dots v^\top T_Kv] \right| \right|  \} \\
&\phantom{ \sup\{  \left|\left| [v^\top T_1v | v^\top T_2v |\dots ] \right| \right| } : v \in \mathbb{R}^D  \& ||v||=1  \\
& = \left|\left| \big[ \left|\left| T_1 \right|\right|_2 ,  \left|\left| T_2 \right|\right| , \dots   \left|\left| T_D \right|\right| \big]  \right|\right| \\
& \le \left|\left| \big[ \left|\left| T_1 \right|\right|_F ,  \left|\left| T_2 \right|\right|_F , \dots   \left|\left| T_D \right|\right|_F \big]  \right|\right| \\
& = \sqrt{ \left( \left|\left| T_1 \right|\right|_F^2 + \left|\left| T_2 \right|\right|_F^2 + \dots + \left|\left| T_D \right|\right|_F \right) } \\
& = ||T||_F \numberthis
\end{align*}
\end{proof}

\begin{lemma}
(Robust Power Method from \cite{MoM})
\label{lemma:rpm}
If $\hat{T} = T + E \in \mathbb{R}^{K \times K \times K}$, where $T$ is an symmetric tensor with orthogonal decomposition $ T = \sum_{k=1}^K{\lambda_k u_k \otimes u_k \otimes u_k}$ with each $\lambda_k >0$, and $E$ has operator norm $||E||_2 \le \epsilon$. Let $\lambda_{\min} = \min_{k=1}^K \{\lambda_k\}$ and $\lambda_{\max} = \max_{k=1}^K \{\lambda_k\}$.
Let there exist constants $c_1,c_2$ such that $\epsilon \le c_1\cdot (\lambda_{\min}/K)$, and $ N \ge c_2( \log{K} + \log{\log{(\lambda_{\max}/\epsilon)}})$. Then if Algorithm 1 in \cite{MoM} is called for $K$ times, with $L = poly(K)\log(1/\eta)$ restarts each time for some $\eta \in (0,1)$, then with probability at least $1-\eta$, there exists a permutation $\pi$ on $[K]$, such that,
 \begin{align*}
 &||u_{\pi(k)}-\hat{u}_k|| \le 8\frac{\epsilon}{\lambda_{\pi(k)}} \text{,    } |\lambda_k - \lambda_{\pi(k)}| \le 5\epsilon \text{ } \forall k \in [K] \numberthis
 \end{align*}
 \end{lemma}

Since $\epsilon \le c_1\cdot (\lambda_{\min}/K)$ and  $ \lambda_k = \frac{1}{\sqrt{\pi_k}}, \forall k \in [K]$, we need,
 \begin{align*}
 N & \ge c_2\left( \log{K} + \log{\log{\left(\frac{K\lambda_{\max}}{c_1\lambda_{\min}}\right)}} \right) \\
 &  \ge c_2\left( \log{K} + \log{\log{\left(\frac{K}{c_1}   \sqrt{\frac{\pi_{max}}{\pi_{min}}} \right)}} \right)
 \numberthis
 \end{align*}

This contributes in the first lower bound ($n_1$) of $N$ in Theorem \ref{thm:bound}.

\section{Tail Inequality}

\begin{lemma}
\label{lemma:tailineq}
If we draw $N$ i.i.d. samples $x_1,x_2 \dots x_N$ through the generative process in Equation \ref{1}  corresponding to $N$ users, and the vectors  probability mass function of the items $y$  estimated from these $N$ samples are $\hat{p}(y)$ whereas the true p.m.f is $p(y)$ with $y \in \{y_1, y_2 \dots y_D\}$ , then with probability at least $1-\delta$ with $\delta \in (0,1)$, 

\begin{align}
\label{eqn:prob1}
\left|\left|\hat {p}(y) - p(y) \right|\right|_F  & \le \frac{2}{\tilde{d}_{1s}\sqrt{N}}\left( 1 + \sqrt{\frac{\log (1 / \delta)}{2}}\right)\\
\label{eqn:prob2}
 \left|\left|\hat {p}(y,y) - p(y, y) \right|\right|_F  & \le \frac{2}{\tilde{d}_{2s}\sqrt{N}}\left( 1 + \sqrt{\frac{\log (1 / \delta)}{2}}\right)\\
\label{eqn:prob3}
\left|\left|\hat {p}(y, y, y) - p(y, y, y) \right|\right|_F & \le \frac{2}{\tilde{d}_{3s}\sqrt{N}}\left( 1 + \sqrt{\frac{\log (1 / \delta)}{2}}\right)
\end{align}

where, $\tilde{d}_{1s} = \frac{1}{N}\sum_{i =1}^N nnz(x_i)$, $\tilde{d}_{2s} = \frac{1}{N}\sum_{i=1}^N nnz(x_i)^2$, $\tilde{d}_{3s} = \frac{1}{N}\sum_{i=1}^N nnz(x_i)^3$, and $nnz(x_i)$ is the non-zero entries in row $x_i$ of the data $X$ as described in section 3.

\end{lemma}

\begin{proof}
The generative process in Equation \ref{1}  results in samples $x_{1:N}$  that are vectors of count data, with $\sum_y [x_u]_d = n_u$, where $x_u$ is the sample corresponding to the user $u$, and $n_u$ is the sum of the counts of all the items for $u$. The operation $\sum_y$ denotes the sum across the dimensions. From here, we can show that $||x_u|| = \sqrt{\sum_y [x_u]_d^2} \le \sum_y [x_u]_d = n_u$, since $[x_u]_d \ge 0, \forall d \in 1,2 \dots D$. Therefore, the samples have bounded norm.

Without loss of generality, if we assume $||x|| \le 1$ $ \forall x \in X$, then from Lemma 7 of supplementary material of \cite{slda}, with probability at least $1-\delta$ with $\delta \in (0,1)$,

\begin{align}
\label{eqn:tail1}
\left| \left| \hat{\mathbb{E}}[x] - \mathbb{E}[x]  \right| \right|_F & \le \frac{2}{\sqrt{N}}\left( 1 + \sqrt{\frac{\log (1 / \delta)}{2}}\right)  \\
\label{eqn:tail2}
\left| \left| \hat{\mathbb{E}}[x \otimes x] - \mathbb{E}[x \otimes x ]  \right| \right|_F & \le \frac{2}{\sqrt{N}}\left( 1 + \sqrt{\frac{\log (1 / \delta)}{2}}\right)  \\
\label{eqn:tail3}
\left| \left| \hat{\mathbb{E}}[x \otimes x \otimes x] - \mathbb{E}[x \otimes x \otimes x]  \right| \right|_F  & \le \frac{2}{\sqrt{N}} \left( 1 + \sqrt{\frac{\log (1 / \delta)}{2}}\right) 
\end{align}
where $\mathbb{E}$ stands for true expectation, and $\mathbb{\hat{E}}$ stands for the expectation estimated from the $N$ samples, i.e.,
\begin{align*}
\hat{\mathbb{E}}[x] &=\frac{1}{N} \sum_{i=1}^N x_i = \frac{1}{N} X^\top \mathbf{1} \\ 
\hat{\mathbb{E}}[x \otimes x] &=\frac{1}{N} \sum_{i=1}^N x_i \otimes x_i = \frac{1}{N} X^\top X\\
\hat{\mathbb{E}}[x \otimes x \otimes x] &=\frac{1}{N} \sum_{i=1}^N x_i \otimes x_i \otimes x_i = \frac{1}{N} X \otimes X \otimes X
\end{align*}

Now, since each of our samples $x_{1:N}$ contains binary data, probability of the items can be estimated from the training data as $\hat{p}(y) = \frac{\mathbb{\hat{E}}[x]}{\sum_y \mathbb{\hat{E}}[x]} $, where $ \sum_y \mathbb{\hat{E}}[x]$ is the sum of $\mathbb{\hat{E}}[x]$ across the dimensions, i.e., all the items. Also, it can be shown that $\sum_y \mathbb{\hat{E}}[x] = \tilde{d}_{1s} $. Therefore  $\hat{p}(y) = \frac{\mathbb{\hat{E}}[x]}{ \tilde{d}_{1s}} $. Please note that $\sum_y \mathbb{E}[x]  \approx \sum_y \mathbb{\hat{E}}[x] = \tilde{d}_{1s}$, and therefore, $\hat{p}(y) - p(y) = \frac{1}{ \tilde{d}_{1s}} ( \mathbb{\hat{E}}[x] - \mathbb{E}[x])$, and using this in Equation \ref{eqn:tail1}, we get the first inequality of the Lemma (Equation \ref{eqn:prob1}).

Since  $\tilde{d}_{2s} =\sum_y \sum_y \mathbb{\hat{E}}[x \otimes x] $ and $ \tilde{d}_{3s} =\sum_y \sum_y \sum_y \mathbb{\hat{E}} [x \otimes x \otimes x] $, the pairwise and triple-wise probability matrices can be estimated as,
 
 \begin{align*}
 \hat{p}(y,y) &= \frac{\mathbb{\hat{E}}[x\otimes x]}{\sum_y \sum_y \mathbb{\hat{E}}[x \otimes x]} =  \frac{\mathbb{\hat{E}}[x\otimes x]}{ \tilde{d}_{2s}}\\
  \hat{p}(y,y,y) &= \frac{\mathbb{\hat{E}}[x\otimes x]}{\sum_y \sum_y \sum_y \mathbb{\hat{E}}[x \otimes x \otimes x]} =  \frac{\mathbb{\hat{E}}[x\otimes x \otimes x]}{ \tilde{d}_{3s}}
  \end{align*}
 
Since  $\sum_y \sum_y \mathbb{E}[x \otimes x]  \approx \sum_y \sum_y \hat{\mathbb{E}}[x \otimes x]  = \tilde{d}_{2s}$, and $\sum_y \sum_y \sum_y \mathbb{E}[x \otimes x \otimes x] \approx \sum_y \sum_y \sum_y \hat{\mathbb{E}}[x \otimes x \otimes x]  =  \tilde{d}_{3s}$, we can establish the following equations,

\begin{align*}
\hat{p}(y,y) - p(y,y)  &= \frac{1}{  \tilde{d}_{2s}}\left( \mathbb{\hat{E}}[x\otimes x] -\mathbb{E}[x\otimes x] \right) \\
\hat{p}(y,y,y) - p(y,y,y)  &= \frac{1}{  \tilde{d}_{3s}}\left( \mathbb{\hat{E}}[x\otimes x \otimes x] -\mathbb{E}[x\otimes x \otimes x] \right)
\end{align*}

Substituting these equations in Equation \ref{eqn:tail2} and \ref{eqn:tail3}, we complete the proof.
\end{proof}

\section{Bounds on the Parameters}

Assigning $\varepsilon = \left( 1 + \sqrt{\frac{\log (1 / \delta)}{2}}\right)$ in the inequalities of Lemma \ref{lemma:tailineq}, we get

 $\varepsilon_{M_2} =\left|\left|\hat {p}(y,y) - p(y, y) \right|\right|_2 \le \left|\left|\hat {p}(y,y) - p(y, y) \right|\right|_F  \le \frac{2\varepsilon}{\tilde{d}_{2s}\sqrt{N}}$, and

\begin{align*} 
\varepsilon_{M_3} &=||M_3 - \hat{M}_3||_2 = \left|\left|\hat {p}(y,y,y) - p(y,y,y) \right|\right|_2 \\
& \le \left|\left|\hat {p}(y,y,y) - p(y,y,y) \right|\right|_F  \le \frac{2\varepsilon}{\tilde{d}_{3s}\sqrt{N}}
\end{align*} since operator norm is smaller than Frobenius norm for both matrices and tensors.

Therefore, to satisfy $ \varepsilon_{M_2} \le \sigma_K(M_2)/2$, we need, 
\begin{align*}
N \ge \Omega\left( \left( \frac{\varepsilon}{\tilde{d}_{2s}\sigma_K(M_2)} \right)^2 \right)
\end{align*}
This contributes in the second lower bound ($n_2$) of $N$ in Theorem \ref{thm:bound}.

From Appendix B in \cite{chaganty2013spectral},

\begin{align*}
\label{eqn:etw}
\varepsilon_{tw} &= ||M_3(W,W,W) - \hat{M}_3(\hat{W},\hat{W},\hat{W})||_2 \\
&\le ||M_3||_2 \left( ||\hat{W}||_2^2 + ||\hat{W}||_2||W||_2 + ||W||_2^2 \right)\varepsilon_W \\
& \phantom{||M_3||_2\left( ||\hat{W}||_2^2 + ||\hat{W}||_2||W||_2 + ||W||_2^2 \right)} + || \hat{W} ||^3\varepsilon_{M_3} \\
& \le ||M_3||_2\frac{(2 + \sqrt{2} + 1)}{\sigma_K(M_2)} \varepsilon_W + \frac{2\sqrt{2}}{\sigma_K(M_2)^{3/2}}\varepsilon_{M_3} \\
& \le ||M_3||_2\frac{(3 + \sqrt{2} )}{\sigma_K(M_2)} \cdot  \frac{2}{\sigma_K(M_2)^{3/2}} \varepsilon_{M2} + \frac{2\sqrt{2}}{\sigma_K(M_2)^{3/2}}\varepsilon_{M_3} \\
& \le \frac{10||M_3||_2}{\sigma_K(M_2)^{5/2}} \cdot \varepsilon_{M2} + \frac{2\sqrt{2}}{\sigma_K(M_2)^{3/2}}\varepsilon_{M_3} \\
& \le \left( \frac{10}{\tilde{d}_{2s}\sigma_K(M_2)^{5/2}}  + \frac{2\sqrt{2}}{\tilde{d}_{3s}\sigma_K(M_2)^{3/2}} \right) \frac{2\varepsilon}{\sqrt{N}} 
\numberthis
\end{align*}

Please note that $||M_3||_2 \le ||M_3||_F \le 1$, because $M_3$ is a tensor with individual elements as probabilities.

From \ref{lemma:rpm}, $\epsilon \le c_1\cdot (\lambda_{\min}/K)$, and we can assign $\epsilon$ as the upper bound of $\varepsilon_{tw}$. To satisfy this, we need, 
\begin{align*}
& \left( \frac{10}{\tilde{d}_{2s}\sigma_K(M_2)^{5/2}}  + \frac{2\sqrt{2}}{\tilde{d}_{3s}\sigma_K(M_2)^{3/2}} \right) \frac{2\varepsilon}{\sqrt{N}} \le c_1\frac{\lambda_{\min}}{K} 
\text{, or,} \\
& \left( \frac{10}{\tilde{d}_{2s}\sigma_K(M_2)^{5/2}}  + \frac{2\sqrt{2}}{\tilde{d}_{3s}\sigma_K(M_2)^{3/2}} \right) \frac{2\varepsilon}{\sqrt{N}} \le c_1\frac{1}{K\sqrt{\pi_{\max}}}
\end{align*}

Since $\pi_{\max} \le 1$, we need

\begin{align*}
N \ge \Omega\left( K^2 \left( \frac{10}{\tilde{d}_{2s}\sigma_K(M_2)^{5/2}}  + \frac{2\sqrt{2}}{\tilde{d}_{3s}\sigma_K(M_2)^{3/2}} \right)^2 \varepsilon^2 \right)
\end{align*}

This contributes to $n_3$ in Theorem \ref{thm:bound}.

Here, we will derive the final bounds for the reconstruction error for the parameters. Since $\mu_k =  W^\dagger u_k $ (Algorithm \ref{alg:mom}), with probability at least $1 - \delta$,
\begin{align*}
&||\mu_k-\hat{\mu}_k||\\
& = ||W^\dagger u_k - \hat{W}^\dagger \hat{u}_k|| \\
&= ||W^\dagger u_k  -W^\dagger \hat{u}_k +W^\dagger \hat{u}_k-\hat{W}^\dagger \hat{u}_k|| \\
& \le ||W^\dagger||_2||u_k - \hat{u}_k|| + ||W^\dagger - \hat{W}^\dagger||_2|| \hat{u}_k|| \\
& \le ||W^\dagger||_2 \frac{8\epsilon}{\lambda_k} + \varepsilon_{W^\dagger}    \\
& \le  8\sqrt{\sigma_1(M_2)}\epsilon + \frac{2 \sqrt{\sigma_1(M_2)} }{\sigma_K \left( M_2 \right)} \varepsilon_{M_2} \\
\numberthis
\end{align*}

Since $\frac{1}{\lambda_k} = \sqrt{\pi_k} \le 1$. Therefore, with probability at least $1 - \delta$,
\begin{align*}
&||\mu_k-\hat{\mu}_k|| \\
& \le  8\sqrt{\sigma_1(M_2)} \left( \frac{10}{\tilde{d}_{2s}\sigma_K(M_2)^{5/2}}  + \frac{2\sqrt{2}}{\tilde{d}_{3s}\sigma_K(M_2)^{3/2}} \right) \frac{2\varepsilon}{\sqrt{N}} \\
&    \phantom{8\sqrt{\sigma_1(M_2)}  \frac{10}{\tilde{d}_{2s}\sigma_K(M_2)^{5/2}} +  \frac{2\sqrt{2}}{111111} }  +   \frac{2 \sqrt{\sigma_1(M_2)} }{\sigma_K \left( M_2 \right)} \frac{2\varepsilon}{\tilde{d}_{2s}\sqrt{N}}\\
& \le   \left( \frac{160\sqrt{\sigma_1(M_2)}}{\tilde{d}_{2s}\sigma_K(M_2)^{5/2}}  + \frac{32\sqrt{2\sigma_1(M_2)}}{\tilde{d}_{3s}\sigma_K(M_2)^{3/2}} + \frac{4 \sqrt{\sigma_1(M_2)} }{\tilde{d}_{2s}\sigma_K \left( M_2 \right)} \right) \frac{\varepsilon}{\sqrt{N}} \\
\numberthis
\end{align*}

\begin{align*}
|\pi_k-\hat{\pi}_k| &=\left| \frac{1}{\lambda_k^2}- \frac{1}{\hat{\lambda}_k^2} \right| = \left| \frac{(\lambda_k+\hat{\lambda}_k)(\lambda_k-\hat{\lambda}_k)}{\lambda_k^2 \hat{\lambda}_k^2} \right| \\
&= \left| \sqrt{\pi_k \hat{\pi}_k} \left( \sqrt{\pi_k} + \sqrt{\hat{\pi}_k}  \right) (\lambda_k- \hat{\lambda}_k) \right| \\
& \le 2 | \lambda_k- \hat{\lambda}_k | \le 10\epsilon
\end{align*}

since $ | \lambda_k- \hat{\lambda}_k | \le 5\epsilon$ from Lemma \ref{lemma:rpm}. Therefore, with probability at least $1 - \delta$, we get
\begin{align*}
|\pi_k-\hat{\pi}_k| \le \left( \frac{200}{\sigma_K(M_2)^{5/2}}+ \frac{40\sqrt{2}}{\sigma_K(M_2)^{3/2}} \right)\frac{\varepsilon}{\tilde{d}_{3s}\sqrt{N}} 
\end{align*}

where $\varepsilon = \left( 1 + \sqrt{\frac{\log (1 / \delta)}{2}}\right)$ all along. This completes the proof of Theorem \ref{thm:bound}.

\end{document}